\documentclass[
    letterpaper,
    11pt,
    sort&compress,
]{article}

\usepackage[colorlinks=true,linkcolor=blue,citecolor = blue]{hyperref}
\usepackage{algorithm,algorithmic}

\usepackage{macros}
\usepackage{bbm}
\usepackage[numbers]{natbib}
\usepackage[margin=1in]{geometry}
\usepackage{graphicx}
\usepackage{empheq}
\usepackage[T1]{fontenc}
\usepackage{enumitem}
\usepackage{svg}
\usepackage{bm}
\usepackage{easybmat}

\title{Near-Optimal Pure Exploration in Matrix Games:\\ A Generalization of Stochastic Bandits \& Dueling Bandits}

\author{Arnab Maiti\thanks{Equal contribution.} \and Ross Boczar\footnotemark[1]\and Kevin Jamieson \and Lillian J. Ratliff
\\\and University of Washington\thanks{
\texttt{\{arnabm2,rjboczar,ratliffl\}@uw.edu}, \texttt{jamieson@cs.washington.edu}}
}
\date{}
\begin{document}

\maketitle
\begin{abstract}
We study the sample complexity of identifying the pure strategy Nash equilibrium  (PSNE) in a two-player zero-sum matrix game with noise.
Formally, we are given a stochastic model where any learner can sample an entry $(i,j)$ of the input matrix $A\in[-1,1]^{n\times m}$ and observe $A_{i,j}+\eta$ where $\eta$ is a zero-mean 1-sub-Gaussian noise. The aim of the learner is to identify the PSNE of $A$, whenever it exists, with high probability while taking as few samples as possible. \citet{zhou2017identify} presents an instance-dependent sample complexity lower bound that depends only on the entries in the row and column in which the PSNE lies.  We design a near-optimal algorithm whose sample complexity matches the lower bound, up to log factors. The problem of identifying the PSNE also generalizes the problem of pure exploration in stochastic multi-armed bandits and dueling bandits, and our result matches the optimal bounds, up to log factors, in both the settings.
\end{abstract}
\section{INTRODUCTION}
Pure exploration is a well-studied topic in stochastic multi-armed bandits. Given a set of arms, the aim is to identify the arm with the highest mean with high probability while sequentially sampling the arms. A practically motivated version of this problem is \emph{dueling bandits} where the learner pulls two arms and observes the winner among them.
The objective is to identify the Condorcet winner: the arm that has probability greater than 1/2 of winning against any other arm. 
The dueling bandit problem has also been studied from the perspective of pure exploration. 
A generalization of dueling bandits are matrix games which model scenarios where multiple strategic agents are involved and are competing against each other in a stochastic environment. 

Consider a hypothetical situation in which firms $F_1$ and $F_2$ both offer the same product. Typically, when both firms set an identical price, the customer base is evenly distributed between them. However, if one firm sets its price higher than the other, it is likely to experience a reduced share of the customer base. This presents a strategic challenge for the firm: should it opt for a lower price $p_1$ or a higher price $p_2$? 
If the firms are non-cooperative, the optimal thing for each firm to do is to set the lower price $p_1$ otherwise they will most probably lose customers. 
Such a strategic scenario in a stochastic environment can be formally captured by a two-player zero-sum game defined by a matrix $A$ as follows. Consider $A=[0,0.25; -0.25,0]$ where $A_{i,j}$ denotes the expected increase in the fraction of customers visiting firm $F_1$ if first decides to charge $p_i$ and $F_2$ decides to charge $p_j$. Both firms setting a low price $p_1$ corresponds to the entry $(1,1)$ and such an entry is referred to as pure strategy Nash equilibrium (PSNE) in game theory. 

Formally, an entry is a PSNE of $A$ if it is the highest entry in its column and lowest entry in its row. Although, a PSNE need not exist in every matrix $A$, there are some conditions under which they exist (see \citet{shapley1964some,radzik1991saddle,duersch2012pure}). For instance, \citet{duersch2012pure} showed that a symmetric two-player zero-sum game has a pure strategy equilibrium if and only if it is not a generalized rock-paper-scissors matrix. In this paper we generalize both stochastic multi-armed bandits and dueling bandits by studying the problem of identifying a PSNE, whenever it exists, in matrix games with noise.

\subsection{Problem Setting}
We consider two problems settings. 

\textbf{Matrix Games.} In this problem setting, there is an arbitrary input matrix $A\in[-1,1]^{n\times m}$ which is unknown to the learner. The learner can sample an entry $(i,j)$ of $A$ and observe the random variable $X_{i,j}=A_{i,j}+\eta$ where $\eta$ is a zero-mean 1-sub-Gaussian noise. The aim is to design a $\delta$-PAC learner, that is, a learner which has the following property: if the learner stops after a finite time $\tau$ and returns an entry $(\widehat i,\widehat j)$ then  $(\widehat i,\widehat j)$ is a PSNE of $A$ with probability at least $1-\delta$. 
The objective is to minimize the number of samples drawn from $A$.

\textbf{Dueling Bandits.} In dueling bandits setting, there are $K$ arms and a matrix $\mathbf{P}\in [0,1]^{K\times K}$. The matrix $\mathbf{P}$ has the property that for all $i,j$ we have $\mathbf{P}_{i,j}+\mathbf{P}_{j,i}=1$. In each round $t$, a pair of arms $i,j$ are pulled (also known as a duel) and a winner $X_t$ is declared where $\mathbb{P}(X_t=i)=1-\mathbb{P}(X_t=j)=\mathbf{P}_{i,j}$. An arm $\istar$ is called a Condorcet winner if for all $j\in [K]\setminus \{\istar\}$ we have $\mathbf{P}_{\istar,j}>1/2$. In pure exploration, one aims to design a learner that identifies the Condorcet winner, whenever it exists, with probability at least $1-\delta$ by performing as few duels as possible.
Note that dueling bandits is a special case of a matrix game with $A = \mathbf{P}$. Here the PSNE of $A$ is equal to $(i_*,i_*)$ where $i_*$ is the Condorcet winner.

\subsection{Contributions}
\citet{zhou2017identify} initiated the study of identifying the PSNE where they showed that any $\delta$-PAC learner requires $\Omega(\mathbf{H}_1\log(1/\delta))$ samples to find the PSNE $(\istar,\jstar)$, whenever it exists, where \[\mathbf{H}_1=\sum_{i\neq \istar}\frac{1}{(A_{\istar,\jstar}-A_{i,\jstar})^2}+\sum_{j\neq \jstar}\frac{1}{(A_{\istar,j}-A_{\istar,\jstar})^2}.\] In this paper, we design a $\delta$-PAC learner (\Cref{alg-meta}) that achieves the sample complexity of $\mathbf{H}_1\log(1/\delta)$, up to $\log$ factors. 

Recall that dueling bandits is a special case  of matrix games. 
\citet{haddenhorst2021testification} showed that any Condorcet $\delta$-PAC learner requires $\sum_{i\neq \istar}\Delta_{i,\istar}^{-2} \log(1/\delta)$ where $\Delta_{i,\istar} = \mathbf{P}_{\istar,i}-1/2$. 
In the notation of matrix games, this sample complexity is equal to $\tfrac{1}{2}\mathbf{H}_1 \log(1/\delta)$.
As our algorithm achieves $\mathbf{H}_1\log(1/\delta)$, up to $\log$ factors, it is a near-optimal Condorcet learner. 
To the best of our knowledge, this is the first algorithm to achieve this for Condorcet winner identification, without making additional strong assumptions (see the discussion on Dueling Bandits in~\Cref{sec:related}).  

Beyond the theoretical contributions, we also benchmark our algorithm against strong baselines in  \Cref{sec:experiments} and demonstrate that our algorithm is also superior empirically. 

\subsection{Related Work}
\label{sec:related}
\textbf{Matrix games.}
As just described \citet{zhou2017identify} showed that identifying the PSNE of a zero-sum matrix game requires a sample complexity of $\Omega(\mathbf{H}_1\cdot\log(1/\delta))$. They also provided an algorithm based on LUCB that achieves an upper bound of $O(\mathbf{H}_1\log(\mathbf{H}_1/\delta)+\frac{nm}{\tilde\Delta})$ where $\tilde\Delta$ is a matrix-dependent parameter. 
This upper bound is sub-optimal, and can be far worse than the lower bound. 
For instance, when $\delta$ and the gaps $|A_{i,j}-A_{\istar,\jstar}|$ are constants, the upper bound scales as $O(nm)$ where as the lower bound is $\Omega(n+m)$. Recently, \citet{maiti2023instance} studied the problem of finding $\varepsilon$-Nash Equilibrium in $n\times 2$ games and provided near-optimal instance-dependent bounds. The instance-dependent parameters include the gaps between the entries of the matrix, and difference between the value of the game and reward received from playing a sub-optimal row. 
Later, \citet{maiti2023logarithmic} extended the techniques of \citet{maiti2023instance} to identify the support of the Nash equilibrium in arbitrary $n\times m$ games. However, the bounds in their paper are sub-optimal. 

An important class of learners for matrix games utilizes no-regret algorithms for adversarial bandits (cf. \citet{freund1999adaptive}). 
Specifically, initialize two independent copies of \expix/ where the first selects rows to maximize reward, and the other selects columns to minimize loss, where both are fed $A_{i,j} + \eta$. 
Due to the fact that both algorithms enjoy an external regret bound of $\sqrt{K T}$ for a game played for $T$ time steps with $K$ arms, it can be shown that with constant probability, the average of the plays converges to an $\epsilon$-approximate Nash equilibrium as soon as $T$ exceeds $\frac{n+m}{\epsilon^2}$, ignoring log factors.
Consequently, this guarantees that the PSNE can be identified with constant probability once $T$ exceeds $\frac{n+m}{\Delta_{\min}^2}$ where $$\Delta_{\min} = \min\{ \min_{i \neq i_*} A_{i_*,j_*}-A_{i,j_*}, \min_{j \neq j_*} A_{i_*,j}-A_{i_*,j_*} \}.$$
The Tsalis-Inf algorithm of \citet{zimmert2021tsallis} is a no-regret algorithm for adversarial bandits that also enjoys some instance-dependent guarantees under certain favorable conditions. 
While we are not aware of any analysis of this algorithm in the matrix game setting, our experiments will show that while it is a stronger baseline than \expix/, its empirical performance is far inferior to our algorithm.
% It can also be easily shown that a no-regret algorithm like \expix/ has an upper bound of $\frac{n+m}{\Delta_{\min}^2}$, up to log factors.
%\kevin{We should also mention that Tsalis-Inf implies a sample compelxity here using no-regret reducton}

\textbf{Stochastic Multi-Armed Bandits.} Best arm identification in the fixed confidence setting is a well-studied topic (see \Cref{appendix:MAB} for the problem formulation). One of the earlier works in \citet{even2002pac} introduces the Successive Elimination algorithm that achieved an upper bound of $\sum_{i\neq \istar}\Delta_i^{-2}\log(n\Delta_i^{-2}/\delta)$ where $\istar$ is the best arm and $\Delta_i=\mu_{\istar}-\mu_{i}$. 
A lower bound of $\sum_{i\neq \istar}\Delta_i^{-2}\log(1/\delta)$ was then established by \citet{mannor2004sample} (later it was refined and simplified by \citet{kaufmann2016complexity}). 
The lower upper confidence bound (LUCB) algorithm designed by \citet{kalyanakrishnan2012pac} achieved an upper bound of $\sum_{i\neq \istar}\Delta_i^{-2}\log(\sum_{j\neq \istar}\Delta_j^{-2}/\delta)$. 
% Then \citet{jamieson2013finding} designed an algorithm called PRISM that achieved an upper bound of $\sum_{i\neq \istar}\Delta_i^{-2}\log(\sum_{j\neq \istar}\Delta_j^{-2})$ or $\sum_{i\neq \istar}\Delta_i^{-2}\log(\sum_{j\neq \istar}\Delta_j^{-2})$ depending on the parameterization of the algorithm. 
Eventually, an upper bound of $\sum_{i\neq \istar}\Delta_i^{-2}\log(\log(\Delta_i^{-2})/\delta)$ was achieved by the Exponential-gap Elimation algorithm of \citet{karnin2013almost} and the lil'UCB algorithm of \citet{jamieson2014lil}. \citet{jamieson2014lil} also showed that this upper bound is optimal for two-armed bandits. Later \citet{chen2017towards} provided even tighter upper bounds in the general case based on entropy-like terms determined by the gaps. 
% \citet{kaufmann2016complexity} also established a generic framework to characterize the lower bounds in multi-armed bandit problems. 

\textbf{Dueling Bandits.} The dueling bandits problem is a well-studied variant of stochastic bandits (cf. \cite{bengs2021preference} for a survey). The aim here is to identify the best arm (according to some rule) by noisy pairwise comparisons. 
This paper focuses on the Condorcet winner: the arm (assuming it exists) that has a probability greater than $1/2$ of beating every other arm.
% Condorcet winner is one of the well-studied concepts of best arm.  
\citet{haddenhorst2021testification} provided a lower bound of $\sum_{i\neq \istar}\Delta_{i,\istar}^{-2} \equiv \tfrac{1}{2}\mathbf{H}_1$ to identify the Condorcet winner $\istar$, whenever it exists, where $\Delta_{i,\istar}=\mathbf{P}_{i_\star,i}-1/2$ is the probability that arm $\istar$ beats arm $i$, minus $1/2$.
They also provided an upper bound of $\frac{K}{\Delta_{\min,all}^2}$, up to log factors, where $\Delta_{\min,all}=\min_{i\neq j}|\mathbf{P}_{i,j}-1/2|$. 

In an effort to hit the lower bound, prior works introduced a number of strong assumptions to make the problem easier. 
The strong assumptions include total order, strong stochastic transitivity (SST) and stochastic triangle inequality (STI).
Under total ordering, \citet{mohajer2017active} achieved an upper bound of $\frac{K \log\log K}{\min_{i \neq i_*} \Delta_{i,\istar}^2}$. 
Finally, \citet{ren2020sample} achieved an upper bound of $\sum_{i\neq \istar}\Delta_{i,\istar}^{-2} \equiv \tfrac{1}{2}\mathbf{H}_1$ that matches the lower bound up to log factors, but under the strong assumptions of the total ordering setting, SST and STI. 
We emphasize that our work does \emph{not} make these strong assumptions, yet still achieves the optimal sample complexity of $\mathbf{H}_1$, up to log factors.

Regret minimization for dueling bandits has also been studied. Recently \citet{saha2022versatile} achieved the optimal regret bound of $\sum_{i\neq \istar}\frac{\log T}{\Delta_{i,\istar}}$ with the help of Tsallis-inf from \citet{zimmert2021tsallis}. This immediately implies an upper bound of $\frac{1}{\min_{i \neq i_*} \Delta_{i,\istar}}\sum_{i\neq \istar}\frac{1}{\Delta_{i,\istar}}$ with constant probability, up to log factors. 
In the notation of matrix games, this sample complexity is at least $\mathbf{H}_1$, and potentially $\Delta_{\min}^{-1}$ times larger than $\mathbf{H}_1$ for certain instances (see the dueling bandit instance of Section~\ref{sec:experiments}).
To the best of our knowledge, this is the best known sample complexity bound in this setting prior to our work.
%\kevin{The above paragraph needs to define the gaps used and be more precise}

%We note that dueling bandits is a special case of matrix games where $A_{i,j}$ is the probability that arm $i$ beats arm $j$. 
%The PSNE of $A$ then corresponds to the Condorcet winner, and the dueling bandit lower bound of \citet{haddenhorst2021testification} is within a factor of two of the matrix game lower bound $\mathbf{H}_1$ of \citet{zhou2017identify}.
%Thus, an upper bound on identifying the PSNE of a matrix game immediately implies an upper bound on identifying the Condorcet winner of dueling bandits.

% commented the above lines out as I wrote a similar thing in section 2.5
\section{FIXED CONFIDENCE NEAR-OPTIMAL ALGORITHM}
Consider a matrix $A\in [-1,1]^{n\times m}$ which has a PSNE $(\istar,\jstar)$. Let $\Delta_{i,j}=|A_{i,j}-A_{\istar,\jstar}|$ for any $i,j$. Let $\mathbf{H_1}=\sum_{i\neq \istar}\frac{1}{\Delta_{i,\jstar}^2}+\sum_{j\neq \jstar}\frac{1}{\Delta_{\istar,j}^2}$. Let us assume that $A_{i,\jstar}< A_{\istar,\jstar}< A_{\istar,j}$ for any $(i,j)\neq (\istar,\jstar)$, otherwise $\mathbf{H_1}$ is not well-defined. Now we state the main result.
\begin{theorem}\label{theorem-main}
    There is an algorithm (\Cref{alg-meta}) that, with probability at least $1-\delta$, takes at most $c_0\cdot \mathbf{H_1}\cdot \log(\frac{nm\log(\mathbf{H_1})}{\delta})\cdot \log (nm)$ samples from the input matrix $A$ and returns the PSNE $(\istar,\jstar)$. Here $c_0$ is an absolute constant.
\end{theorem}

%\textcolor{red}{ARNAB: $\log^2(1/\delta)$ term is undesirable for designing a fixed budget algorithm for two reasons. First, the error probability would scale as roughly $\exp(-\sqrt{T/H_1})$ instead of $\exp(-T/H_1)$. Second, we would still require the knowledge of $H_1$.}

For the sake of simplicity of presentation, let us assume that $n$ and $m$ are powers of two.
Let $\Delta_g:=\left(\frac{1}{n+m-2}\left(\sum_{i\neq \istar}\frac{1}{\Delta_{i,\jstar}^2}+\sum_{j\neq \jstar}\frac{1}{\Delta_{\istar,j}^2}\right)\right)^{-1/2}$. 
Our goal is to identify $(\istar,\jstar)$ with a sample complexity that scales as $(n+m-2) \Delta_g^{-2} = \mathbf{H}_1$, ignoring log factors. 

We divide the analysis into multiple parts. First, in \Cref{sec:gap}, we design an algorithm using a guess of $\Delta_g$. Next, in \Cref{sec:mid}, we describe the details of the procedures \cmidval/ and \rmidval/ that were used in \Cref{sec:gap}. In \Cref{sec:nogap}, we state the main algorithm which has a sample complexity of $(n+m-2) \Delta_g^{-2}$, ignoring log factors. Finally, in \Cref{sec:dueling}, we address the special case of dueling bandits. In the interest of space, we refer the reader to \Cref{appendix:omit} for a few omitted  calculations.

\subsection{Algorithm with a guess of $\Delta_g$}\label{sec:gap}
%\kevin{You need to explicitly say what \cmidval/ does (i.e., its outputs). This proof is weird because you're claiming results of \cmidval/ but at the point in the proof, we don't even know what it does}

In this section, we aim to design an algorithm to identify the PSNE $(\istar,\jstar)$ when we are given a guess of the parameter $\Delta_g$. 
Informally, the algorithm proceeds in a logarithmic number of stages. 
In each stage it aims to eliminate either half of the sub-optimal rows or sub-optimal columns. 
Suppose at one particular stage the algorithm aims to eliminate half of the sub-optimal rows. 
The algorithm takes samples from each column and computes a value near the median of the column entries. 
Then the algorithm chooses the column $\widehat j$ which has the lowest such value. 
Then the algorithm samples all the entries of the column $\widehat j$ sufficiently and removes those rows whose corresponding entries in column $\widehat j$ are lower than median of the column entries. 

Before formally describing our algorithm, we state the guarantees of the two subroutines \cmidval/ and \rmidval/ from \Cref{sec:mid} that will be used in our algorithm. Consider $\varepsilon,\delta>0$.
Given a set of $n$ arms $\calA$ with means $\mu_1\geq \mu_2\geq \cdots \geq \mu_n$, \cmidval/ outputs a value $\widehat v \in [\mu_{n/2}-\varepsilon,\mu_{n/4+1}+\varepsilon]$ with probability $1-\delta$ by using at most $O(\frac{1}{\varepsilon^2}\log\left(\frac{1}{\delta}\right))$ samples. Similarly, \rmidval/ outputs a value $\widehat v \in [\mu_{3n/4}-\varepsilon,\mu_{n/2+1}+\varepsilon]$ with probability $1-\delta$ by using at most $O(\frac{1}{\varepsilon^2}\log\left(\frac{1}{\delta}\right))$ samples. 
We refer the reader to \Cref{alg-matrix-opt} for a formal description of our algorithm.

\begin{algorithm}[t!]
\caption{\infalg/($A,\Delta,\delta$)}
\begin{algorithmic}[1]
%\STATE \textbf{Require:} Gap Parameter $\Delta_g$
\STATE $\calX_1\gets$ rows of $A$, $\calY_1\gets$ columns of $A$
\FOR{$t=1,2,\ldots$}
\IF{$\max\{|\calX_t|,|\calY_t|\}=2$}
\STATE Sample every element in $\calX_t\times\calY_t$ $\frac{n+m-2}{2}\cdot\frac{50\log(16/\delta)}{\Delta^2}$ times and return the PSNE of the sub-matrix formed by  $\calX_t\times\calY_t$.
\ELSIF{$|\calX_t|\geq |\calY_t|$}
\STATE Let $\calC_j$ denote the elements of column $j$ in the sub-matrix formed by  $\calX_t\times\calY_t$.
\STATE $\varepsilon_t \gets \frac{1}{9}\left(\frac{|\calX_t|}{n+m-2}\right)^{1/2}\cdot\Delta$
\STATE Set $\widehat v_{t,x}(j)\gets \cmidval/(\calC_j,\varepsilon_t,\frac{\delta}{2m^2n^2})$ for each $j\in\calY_t$.
\STATE $\widehat j\gets \arg\min_{j\in\calY_t}\widehat v_{t,x}(j)$ and $\calY_{t+1}\gets \calY_t$
\STATE Sample every entry in $\{(i,\widehat j):i\in \calX_t\}$  $\frac{n+m-2}{|\calX_t|}\cdot\frac{162\log(4n^2m^2/\delta)}{\Delta^2}$ times and compute the empirical average $\bar A_{i,j}$ of these samples. 
\STATE $\calX_{t+1}\gets \{i_1,\ldots,i_{|\calX_t|/2}\}\subset\calX_t=\{i_1,\ldots,i_{|\calX_t|}\}$ where $\bar A_{i_1,\widehat j}\geq\cdots\geq \bar A_{i_{|\calX_t|},\widehat j}$.
\ELSIF{$|\calX_t|<|\calY_t|$}
\STATE Let $\calR_i$ denote the elements of row $i$ in the sub-matrix formed by  $\calX_t\times\calY_t$.
\STATE $\varepsilon_t \gets \frac{1}{9}\left(\frac{|\calY_t|}{n+m-2}\right)^{1/2}\cdot\Delta$
\STATE Set $\widehat v_{t,y}(i)\gets \rmidval/(\calR_i,\varepsilon_t,\frac{\delta}{2m^2n^2})$, for each $i\in\calX_t$.
\STATE $\widehat i\gets \arg\max_{i\in\calX_t}\widehat v_{t,y}(i)$ and $\calX_{t+1}\gets \calX_t$
\STATE Sample every entry in $\{(\widehat i, j):j\in \calY_t\}$ $\frac{n+m-2}{|\calY_t|}\cdot\frac{162\log(4n^2m^2/\delta)}{\Delta^2}$ times and compute the empirical average $\bar A_{i,j}$ of these samples. 
\STATE $\calY_{t+1}\gets \{j_1,\ldots,j_{|\calY_t|/2}\}\subset\calX_t=\{j_1,\ldots,j_{|\calY_t|}\}$ where $\bar A_{\widehat i,j_1}\leq\cdots\leq \bar A_{\widehat i, j_{|\calY_t|}}$.
\ENDIF
\ENDFOR
\end{algorithmic}
\label{alg-matrix-opt}
\end{algorithm}

Now we begin the analysis of \Cref{alg-matrix-opt}.
First, we state the following proposition that provides a lower bound on the gaps $\Delta_{i,\jstar}$.
The lower bound later allows us to argue that a constant fraction of the entries in the optimal column $\jstar$ is well below $A_{\istar,\jstar}$, and this fact plays a crucial role in lowering the number of samples required in each stage.
\begin{proposition}\label{prop:deltag-1}
    Consider a subset $S=\{i_1,\ldots,i_\ell\}\subseteq [n]$ such that $\Delta_{i_1,\jstar}\leq\cdots\leq \Delta_{i_\ell,\jstar}$. Let $\ell \geq 2$.  Then for any $s\geq \lceil\ell/4\rceil+1$, $\Delta_{i_s,\jstar}\geq \frac{1}{2}\left(\frac{|S|}{n+m-2}\right)^{1/2}\Delta_g$.
\end{proposition}

\begin{proof}
    Consider an index $s\geq \lceil\ell/4\rceil+1$. Now we have the following:
    \begin{align*}
        \left(\frac{|S|}{n+m-2}\right)^{-1}\Delta_g^{-2}&=\frac{1}{|S|}\left(\sum_{i\neq \istar}\frac{1}{\Delta_{i,\jstar}^2}+\sum_{j\neq \jstar}\frac{1}{\Delta_{\istar,j}^2}\right)\\
        &\geq \frac{1}{|S|}\sum_{i\in S\setminus\{\istar\}}(\Delta_{i,\jstar})^{-2}\\
        &\geq \frac{1}{|S|}\sum_{k=2}^{\lceil\ell/4\rceil+1}(\Delta_{i_k,\jstar})^{-2} \\
        &\geq \frac{\lceil\ell/4\rceil}{|S|}\cdot \Delta_{i_s,\jstar}^{-2}\\
        %\tag{as $\Delta_{i_s,\jstar}\geq\Delta_{i_k,\jstar}$ for all $k\leq \lceil\ell/4\rceil +1$}\\
        &\geq(2\Delta_{i_s,\jstar})^{-2},
        %\tag{as $|S|=\ell$}
    \end{align*}
    where the second to last inequality holds since $\Delta_{i_s,\jstar}\geq\Delta_{i_k,\jstar}$ for all $k\leq \lceil\ell/4\rceil +1$, and the last inequality is due to the fact that $|S|=\ell$. 
\end{proof}
Similarly, we  have the following proposition.
\begin{proposition}\label{prop:deltag-2}
    Consider a subset $S=\{j_1,\ldots,j_\ell\}\subseteq [m]$ such that $\Delta_{\istar,j_1}\leq\cdots\leq \Delta_{\istar,j_\ell}$. Let $\ell\geq 2$.  Then for any $s\geq \lceil\ell/4\rceil+1$, $\Delta_{\istar,j_s}\geq \frac{1}{2}\left(\frac{|S|}{n+m-2}\right)^{1/2}\Delta_g$.
\end{proposition}

Next, we establish the sample complexity of our algorithm in the following lemma.
\begin{lemma}\label{lem:psne-sample}
    The sample complexity of the procedure \infalg/($A,\Delta,\delta$) is $c\cdot \frac{n+m-2}{\Delta^2}\cdot \log\left( \frac{nm}{\delta}\right)\cdot \log(nm)$  where $c$ is an absolute constant.
\end{lemma}
\begin{proof}
Let $\calT_1:=\{t:|\calX_{t+1}|=\frac{|\calX_t|}{2}\}$ and $\calT_2:=\{t: |\calY_{t+1}|=\frac{|\calY_t|}{2}\}$, and consider an iteration $t\in \calT_1$.
Recall that $\varepsilon_t= \frac{1}{9}\left(\frac{|\calX_t|}{n+m-2}\right)^{1/2}\cdot\Delta$.
Due to \Cref{midval:sample}, $\cmidval/(\calC_j,\varepsilon_t,\frac{\delta}{2m^2n^2})$ takes $\frac{c\cdot\log(2n^2m^2/\delta)}{\varepsilon_t^2}$ samples where $c$ is an absolute constant.
Hence, the total number of samples taken in the iteration $t$ is upper bounded by 
$|\calY_t|\cdot \tfrac{n+m-2}{|\calX_t|}\cdot\tfrac{81c\cdot\log(2n^2m^2/\delta)}{\Delta^2}
   +|\calX_t|\cdot \frac{n+m-2}{|\calX_t|}\cdot\tfrac{162\log(4n^2m^2/\delta)}{\Delta^2}\leq (c+1)\cdot (n+m-2)\cdot\tfrac{162\log(4n^2m^2/\delta)}{\Delta^2}$, since
 $|\calY_t|\leq |\calX_t|$. By an identical argument, this bound applies to each iteration $t\in \calT_2$.

% Similarly consider an iteration $t\in \calT_2$. Recall that $\varepsilon_t= \frac{1}{9}\left(\frac{|\calY_t|}{n+m-2}\right)^{1/2}\cdot\Delta$. Due to \Cref{midval:sample}, $\textsc{MidVal}(\calR_j,\texttt{ row },\varepsilon_t,\frac{\delta}{2m^2n^2})$ takes $\frac{c\cdot\log(2n^2m^2/\delta)}{\varepsilon_t^2}$ samples where $c$ is an absolute constant. Hence, the total number of samples taken in this iteration is upper bounded by $|\calX_t|\cdot \frac{n+m-2}{|\calY_t|}\cdot\frac{81c\cdot\log(2n^2m^2/\delta)}{\Delta^2}+|\calY_t|\cdot \frac{n+m-2}{|\calY_t|}\cdot\frac{162\log(4n^2m^2/\delta)}{\Delta^2}\leq (c+1)\cdot (n+m-2)\cdot\frac{162\log(4n^2m^2/\delta)}{\Delta^2}$. We get the last inequality as $|\calY_t|> |\calX_t|$.

Finally consider iteration $t$ such that $\max\{|\calX_t|,|\calY_t|\}=2$. The total number of samples taken in this iteration $t$ is upper bounded by $4\cdot\frac{n+m-2}{2}\cdot\frac{50\log(16/\delta)}{\Delta^2}$.

We now get the desired sample complexity as $|\calT_1|=\log_2(n)-1$ and $|\calT_2|=\log_2(m)-1$. 
\end{proof}
Finally, we establish the correctness of our algorithm in the following lemma. The proof of the lemma makes use of the fact that the entries around the median of column $\jstar$ are far below $A_{\istar,\jstar}$. The sampling done from each column ensures that the entries around the median of column $\widehat j$ are also far below $A_{\istar,\jstar}$. This along with the fact that $A_{\istar,\widehat j}\geq A_{\istar,\jstar}$ ensures that we don't eliminate the row $\istar$.
\begin{lemma}\label{lem:psne-correct}
    If $\Delta\leq \Delta_g$, then the procedure \infalg/(A,$\Delta,\delta$) returns $(\istar,\jstar)$ with probability at least $1-\delta$.
\end{lemma}
\begin{proof}
    Let us consider the case when $\min\{n,m\}\geq 4$. The other cases can be proved analogously. 
    Let \[\calT_1:=\{t:|\calX_{t+1}|=\tfrac{|\calX_t|}{2}\}\ \ \text{and}\ \ \calT_2:=\{t: |\calY_{t+1}|=\tfrac{|\calY_t|}{2}\}.\]
    Consider an iteration $t\in \calT_1$ such that $|\calX_t|\geq 4$ and let us assume that $\istar\in\calX_t$ and $\jstar\in\calY_t$. Observe that $\calY_{t+1}=\calY_t$. We now show that with probability at least $1-\frac{\delta}{nm}$, we have $\istar\in\calX_{t+1}$. Let $\Delta_t=\left(\frac{|\calX_t|}{n+m-2}\right)^{1/2}\Delta_g$. Recall the definition of $\varepsilon_t$ and observe that $\varepsilon_t\leq\Delta_t/9$.

    Consider an index $j\in\calY_t$. Let us relabel the indices in $\calX_t$ as $\{i_1^{(j)},\ldots,i_{|\calX_t|}^{(j)}\}$ such that $A_{i_1^{(j)},j}\geq\ldots \geq A_{i_{|\calX_t|}^{(j)},j}$. Let $G_j$ be the event that $\widehat v_{t,x}(j)\in[A_{i_{|\calX_t|/2}^{(j)},j}-\varepsilon_t,A_{i_{|\calX_t|/4+1}^{(j)},j}+\varepsilon_t]$. Due to \Cref{midval:correct1}, event $G_j$ holds with probability at least $1-\frac{\delta}{2m^2n^2}$.

    Let us assume that $G_j$ holds for all $j\in \calY_t$. This happens with with probability at least $1-\frac{\delta}{2n^2m}$ due to a union bound. As event $G_{\jstar}$ holds, we have the following due to \Cref{prop:deltag-1}:
    \begin{align*}
        \widehat v_{t,x}(\jstar)&\leq A_{i_{|\calX_t|/4+1}^{(\jstar)},\jstar}+\frac{\Delta_t}{9},\\
        &\leq A_{\istar,\jstar}-\frac{\Delta_t}{2}+\frac{\Delta_t}{9}=A_{\istar,\jstar}-\frac{7\Delta_t}{18}.
    \end{align*}
    Recall the definition of $\widehat j$. Since event $G_{\widehat j}$ holds, we have the following for all $\frac{|\calX_t|}{2}\leq s\leq |\calX_t|$:
    \begin{equation*}
        A_{i^{(\widehat j)}_s,\widehat j}\leq \widehat v_{t,x}(\widehat j)+\frac{\Delta_t}{9}\leq \widehat v_{t,x}(\jstar)+\frac{\Delta_t}{9}\leq A_{\istar,\jstar}-\frac{5\Delta_t}{18}.
    \end{equation*}

    Let us assume that for all $i\in \calX_t $, $|\bar A_{i,\widehat j}-A_{i,\widehat j}|\leq \frac{\Delta_t}{9}$. This happens with probability at least $1-\frac{\delta}{2nm^2}$ due to the sub-Gaussian tail bound and union bound. We now have the following for all $\frac{|\calX_t|}{2}\leq s\leq |\calX_t|$:
    \begin{equation*}
       \bar A_{i^{\widehat j}_s,\widehat j}\leq A_{i^{\widehat j}_s,\widehat j} +\frac{\Delta_t}{9 }\leq A_{\istar,\jstar}-\frac{3\Delta_t}{18}.
    \end{equation*}
    This implies that $|\{i\in\calX_t:\bar A_{i,\widehat j}\leq  A_{\istar,\jstar}-\frac{3\Delta_t}{18}\}|\geq \frac{|\calX_t|}{2}$. Recall that $\calX_{t+1}=\{i_1,\ldots,i_{|\calX_t|/2}\}\subset\calX_t=\{i_1,\ldots,i_{|\calX_t|}\}$ where $\bar A_{i_1,\widehat j}\geq\cdots\geq \bar A_{i_{|\calX_t|},\widehat j}$. Now observe that $\bar A_{\istar,\widehat j}\geq A_{\istar,\widehat j}-\frac{\Delta_t}{9}\geq A_{\istar,\jstar}-\frac{\Delta_t}{9}$. Hence, $\istar\in\calX_{t+1}$.

    Analogously we can show that for an iteration $t\in \calT_2$ such that $|\calY_t|\geq 4$, $\istar\in\calX_t$ and $\ystar\in\calY_t$, with probability at least $1-\frac{\delta}{nm}$ we have $\istar\in\calX_{t+1}$ and $\ystar\in\calY_{t+1}$.

    Observe that the algorithm terminates at a fixed iteration $t_\star=\log_2(nm)-1$.
    Let $p_t=\bbP(\istar\in \calX_{t},\jstar\in\calY_t|\istar\in \calX_{t-1},\jstar\in\calY_{t-1})$.
    Now we have the following due to the chain rule and Bernoulli's inequality,
    \begin{equation*}
\bbP(\istar\in\calX_{\tstar},\jstar\in\calY_{\tstar})=\prod_{t=2}^{t_\star}p_t\geq \left(1-\tfrac{\delta}{nm}\right)^{t_\star}\geq 1-\delta/2.
    \end{equation*}
    Hence, with probability at least $1-\delta/2$ we have an iteration $t_\star$ such that $\calX_{t_\star}= \{\istar,\widehat i\}$ and $\calY_{t_\star}= \{\jstar,\widehat j\}$. Under the assumption of the existence of such an iteration $t_\star$, we now show that with probability at least $1-\delta/2$ we return $(\istar,\jstar)$. Let $\Delta_{t_\star}=\left(\frac{|\calX_{t_\star}|}{n+m-2}\right)^{1/2}\Delta_g$. For all $(i,j)\in \calX_{t_\star}\times \calY_{t_\star}$, let $\bar A_{i,j}$ denote the empirical average of the $\frac{n+m-2}{|\calX_{t_\star}|}\cdot\frac{50\log(16/\delta)}{\Delta_g^2}$ samples. Let us assume that for all $(i,j)\in \calX_{t_\star}\times \calY_{t_\star} $, $|\bar A_{i,j}-A_{i,j}|\leq \frac{\Delta_{t_\star}}{5}$. This happens with probability at least $1-\delta/2$ due to sub-gaussian tail bound and union bound. Now observe that $A_{\istar,\jstar}-\frac{\Delta_{t_\star}}{5}\leq \bar A_{\istar,\jstar}\leq A_{\istar,\jstar}+\frac{\Delta_{t_\star}}{5}$. Due to \Cref{prop:deltag-2}, we have $\bar A_{\istar,\widehat j}\geq A_{\istar,\widehat j}-\frac{\Delta_{t_\star}}{5}\geq A_{\istar,\jstar}+\frac{\Delta_{t_\star}}{2}-\frac{\Delta_{t_\star}}{5}>\bar A_{\istar,\jstar}$. Similarly due to \Cref{prop:deltag-1}, we have $\bar A_{\widehat i,\jstar}\leq A_{\widehat i,\jstar}+\frac{\Delta_{t_\star}}{5}\leq A_{\istar,\jstar}-\frac{\Delta_{t_\star}}{2}+\frac{\Delta_{t_\star}}{5}<\bar A_{\istar,\jstar}$. Hence, we return $(\istar,\jstar)$.
\end{proof}

\subsection{\cmidval/ \& \rmidval/  Subroutine}\label{sec:mid}
Here, we define \cmidval/ (\Cref{alg-midval}) which solves the subproblem of finding a value near the median of means of $n$ arms; \rmidval/ is defined analogously in \Cref{appendix:rmidval} along with its guarantees.

Consider $\varepsilon,\delta>0$.
Given a set of $n$ arms $\calA$ with means $\mu_1\geq \mu_2\geq \cdots \geq \mu_n$, we show that \cmidval/ outputs a value $\widehat v \in [\mu_{n/2}-\varepsilon,\mu_{n/4+1}+\varepsilon]$ with probability $1-\delta$ by using at most $O(\frac{1}{\varepsilon^2}\log\left(\frac{1}{\delta}\right))$ samples. We assume that $n$ is a multiple of $4$ for simplicity.

\begin{algorithm}[ht]
\caption{\cmidval/($\calA,\varepsilon,\delta$)}
\begin{algorithmic}[1]
\STATE $\ell\gets \lceil 14\log(\frac{1}{\delta})\rceil$, $\delta_1\gets 0.05$, $\delta_2\gets 0.05$, $k\gets \left\lceil 108\log(\frac{4}{\delta_2})\right\rceil$ and $z\gets \frac{k}{3}+1$.
\FOR{$i=1$ to $\ell$}
\STATE $\calB_i \gets \{k $ arms sampled independently and uniformly at random from $\calA \}$.
\STATE Sample each arm $j\in \calB_i$ $\left\lceil\frac{2\log(\frac{2k}{\delta_1})}{\varepsilon^2}\right\rceil$ times and compute the empirical mean $\widehat \mu_j$.
\STATE $v_i\gets $ $z$-th highest value in $\{\widehat \mu_j:j\in\calB_i\}$
\ENDFOR
\STATE \textbf{return} the median of $\{v_i:i\in[\ell]\}$
\end{algorithmic}
\label{alg-midval}
\end{algorithm}

We first establish the sample complexity of \cmidval/ in the following lemma.
\begin{lemma}\label{midval:sample} \cmidval/($\calA,\varepsilon,\delta$) requires at most $\frac{c\log(1/\delta)}{\varepsilon^2}$ samples.
\end{lemma}
\begin{proof}
As $k,\delta_1$ are constants, the total number of samples is
\[
\left\lceil 14\log\left(\frac{1}{\delta}\right)\right\rceil \cdot k \cdot \left\lceil\frac{2\log(\frac{2k}{\delta_1})}{\varepsilon^2}\right\rceil \leq \frac{c\log(1/\delta)}{\varepsilon^2} 
\]
for an absolute constant $c$.
\end{proof}
Next, we establish the correctness of \cmidval/. The proof follows from the standard median of means argument.
\begin{lemma}\label{midval:correct1}
    Given $n$ arms $\calA$ with means $\mu_1\geq \mu_2\geq \cdots \geq \mu_n$, \cmidval/($\calA,\varepsilon,\delta$) outputs a value $\widehat v \in [\mu_{n/2}-\varepsilon,\mu_{n/4+1}+\varepsilon]$ with probability at least $1-\delta$.
\end{lemma}
\begin{proof}
Consider an iteration $i$. Let $n_{1,i}$ and $n_{2,i}$ be the number of times arms $1$ through $n/4$ and arms $n/4+1$ through $n/2$ were drawn for $\calB_i$, respectively; that is, the counts of occurrences corresponding to the first and second quantiles of highest arm means.
We then say the event $E_i$ holds if 
\begin{itemize}[topsep=-1pt]
    \item $\frac{7k}{40}\leq n_{1,i}\leq \frac{k}{3}$ and $\frac{7k}{40}\leq n_{2,i} \leq \frac{k}{3}$, and 
    \item for all $j\in\calB_i$, $|\widehat \mu_j-\mu_j|\leq \varepsilon$.
\end{itemize}

By a standard sub-Gaussian tail bound and union bound, we have that $|\widehat \mu_j-\mu_j|\leq \varepsilon$ simultaneously for all $j\in \calB_i$ with probability at least $1-\delta_1$. This implies that any sampled arm not in the first quantile has an empirical mean of at most $\mu_{n/4+1}+\varepsilon$. Similarly any sampled arm from the first and second quantile has an empirical mean of at least $\mu_{n/2}-\varepsilon$. 
Next, observe that $\bbE[n_{1,i}]=\bbE[n_{2,i}]=k/4$.
Using the Chernoff bound and union bound, we have $\frac{7k}{40}\leq|n_{1,i}|\leq \frac{k}{3}$ and $\frac{7k}{40}\leq|n_{2,i}|\leq \frac{k}{3}$ with probability at least $1-\delta_2$.
Thus, by a union bound, $\bbP[E_i]\geq 1-\delta_1-\delta_2=\frac{9}{10}$.

We first claim that if $E_i$ holds, then $v_i\in [\mu_{n/2}-\varepsilon,\mu_{n/4+1}+\varepsilon]$.
We see that the $z$-th highest empirical mean is in the interval $[\mu_{n/2}-\varepsilon,\mu_{n/4+1}+\varepsilon]$ if there are \emph{at most} $z-1$ empirical means greater than $\mu_{n/4+1}+\varepsilon$ and there are \emph{at least} $z$ empirical means greater than or equal to $\mu_{n/2}-\varepsilon$.
Noting that $z=\frac{k}{3}+1$, on $E_i$, the first condition occurs as $n_{1,i}\leq k/3 = z-1$, and the second condition occurs as $n_{1,i} + n_{2,i} \geq \frac{7k}{20} > z$.

Finally, let $X_i$ be the indicator variable $\mathbbm{1}_{E_i}$, and let $Z=\sum_{i=1}^\ell X_i$ be the total number of times $E_i$ occurs.
If $Z>\ell/2$, then the median of $\{v_i:i\in[\ell]\}$ lies in $[\mu_{n/2}-\varepsilon,\mu_{n/4+1}+\varepsilon]$.
Observing that $\bbE[Z]\geq \frac{9\ell}{10}$, where $\ell=\lceil14\log(\frac{1}{\delta})\rceil$, by the Chernoff bound we have $\bbP[Z\geq 0.54\ell]\geq 1-\delta$.
Hence, \cmidval/($\calA,\varepsilon,\delta$) returns a value $\widehat v \in [\mu_{n/2}-\varepsilon,\mu_{n/4+1}+\varepsilon]$ with probability at least $1-\delta$.
\end{proof}

\subsection{Algorithm without the knowledge of $\Delta_g$}\label{sec:nogap}
First we describe the guarantees of a procedure based on stochastic bandit algorithms that is used to verify whether an entry $(\widehat i,\widehat j)$ is a PSNE of $A$ or not. 

\begin{lemma}\label{verify-guarantees}
    Given an entry $(\widehat i,\widehat j)$ and parameters $\Delta$ and $\delta$, there is a procedure \textsc{Verify} possessing the following properties with probability $1-\delta$:
\begin{itemize}[itemsep=0pt,topsep=-2pt]
    \item If $\Delta \leq \Delta_g$ and $(\widehat i,\widehat j)=(\istar,\jstar)$, then $(\widehat i,\widehat j)$ is ``accepted'' as the PSNE of $A$ by \textsc{Verify}.
    \item If $(\widehat i,\widehat j)\neq (\istar,\jstar)$, then $(\widehat i,\widehat j)$ is ``not accepted'' as the PSNE of $A$ by \textsc{Verify}.
\end{itemize}
Moreover, \textsc{Verify} terminates after taking $c\cdot \mathbf{H_\star}\cdot \log(\frac{\log(\mathbf{H_\star})}{\delta})$ samples from $A$ where $c$ is an absolute constant and $\mathbf{H_\star}=\frac{n+m-2}{\Delta^2}$.
\end{lemma}
We refer the reader to \Cref{appendix:verify} for more details.

%Now, we describe a meta-algorithm (\Cref{alg-meta}) that finds 
\Cref{alg-meta} is a meta-algorithm for finding $(\istar,\jstar)$ with high probability without the knowledge of $\Delta_g$. \Cref{theorem-main} establishes the guarantees for \Cref{alg-meta}.
\begin{algorithm}
\caption{}
\begin{algorithmic}[1]
\STATE \textbf{Given:} Probability error term $\delta$.
\FOR{$t=1,2,\ldots$}
\STATE $\Delta_t\gets \frac{1}{2^{t-1}}$ and $\delta_t\gets \frac{\delta}{4t^2}$
\STATE $(i_t,j_t)\gets$\infalg/($A,\Delta_t,\delta_t$)
\IF{\textsc{Verify}($A,i_t,j_t,\delta_t,\Delta_t$)=``accepted''}
\STATE \textbf{Return} $(i_t,j_t)$ as the PSNE
\ENDIF
\ENDFOR
\end{algorithmic}
\label{alg-meta}
\end{algorithm}

%We now establish the guarantees of the meta-algorithm (\Cref{alg-meta}) by proving \Cref{theorem-main}.
\begin{proof}[Proof of \Cref{theorem-main}]
    Recall that $\Delta_g=(\frac{n+m-2}{\mathbf{H_1}})^{1/2}$. Consider an iteration $t\leq \lceil \log_2(1/\Delta_g) \rceil+1$. If $(i_t,j_t)\neq (\istar,\jstar)$ then, due to \Cref{verify-guarantees}, $(i_t,j_t)$ is not returned as the PSNE with probability at least $1-\frac{\delta}{4t^2}$.  Consider the iteration $t_\star=\lceil \log_2(1/\Delta_g) \rceil+2$. Observe that $\Delta_{t_\star}\leq \frac{1}{2^{\log_2(1/\Delta_g)}}=\Delta_g$. Due to \Cref{lem:psne-correct}, with probability at least $1-\frac{\delta}{4t_\star^2}$, we have $(i_{t_\star},j_{t_\star})=(\istar,\jstar)$. Due to \Cref{verify-guarantees}, the pair $(\istar,\jstar)$ is accepted and returned as a PSNE with probability at least $1-\frac{\delta}{4t_\star^2}$. 

    Hence, with probability at least $1-\frac{\delta}{4t_\star^2}-\sum_{t=1}^\infty\frac{\delta}{4t^2}>1-\delta$ the algorithm terminates by the iteration $t_\star=\lceil \log_2(1/\Delta_g) \rceil+2$ and returns $(\istar,\jstar)$ as the PSNE.

    Let us assume that the algorithm terminates by the iteration $t_\star=\lceil \log_2(1/\Delta_g) \rceil+2$. Due to \Cref{lem:psne-sample} and \Cref{verify-guarantees}, the number of samples in any iteration $t$ is upper bounded by $c\cdot \frac{n+m-2}{\Delta_t^2}\cdot \log\left(\frac{nm\log(\frac{1}{\Delta_t})}{\delta_t}\right)\cdot \log (nm)$ where $c$ is an absolute constant. The sample complexity $\tau$ of \Cref{alg-meta} is upper bounded as follows:
        \begin{align*}
        \tau &\leq \sum_{t=1}^{t_\star} c\cdot \frac{n+m-2}{\Delta_t^2}\cdot \log\left(\frac{nm\log(\frac{1}{\Delta_t})}{\delta_t}\right)\cdot\log (nm),\\
        &\leq \sum_{t=1}^{t_\star} c\cdot 4^t(n+m-2)\cdot \log\left(\frac{4t^3nm}{\delta}\right)\cdot\log (nm),\\
         &\leq c\cdot (n+m-2)\cdot \log\left(\frac{4t_\star^3\cdot nm}{\delta}\right)\cdot\log (nm)\cdot \sum_{t=1}^{t_\star}4^t,
    \end{align*}
    where the second inequality holds since $\Delta_t=\frac{1}{2^{t-1}}$ and $\delta_t=\frac{\delta}{4t^2}$.
    Using the fact that $\sum_{t=1}^{t_\star}4^t\leq \frac{4^{t_\star+1}}{3}$, we further deduce that
    \begin{align*}
       % \tau &\leq \sum_{t=1}^{t_\star} c\cdot \frac{n+m-2}{\Delta_t^2}\cdot \log\left(\frac{nm\log(\frac{1}{\Delta_t})}{\delta_t}\right)\cdot\log (nm)\\
       % &\leq \sum_{t=1}^{t_\star} c\cdot 4^t(n+m-2)\cdot \log\left(\frac{4t^3nm}{\delta}\right)\cdot\log (nm) \tag{as $\Delta_t=\frac{1}{2^{t-1}}$ and $\delta_t=\frac{\delta}{4t^2}$}\\
       % &\leq c\cdot (n+m-2)\cdot \log\left(\frac{4t_\star^3\cdot nm}{\delta}\right)\cdot\log (nm)\cdot \sum_{t=1}^{t_\star}4^t\\
       \tau &\leq \frac{4c}{3}\cdot (n+m-2)\cdot \log\left(\frac{4t_\star^3\cdot nm}{\delta}\right)\cdot\log (nm)\cdot 2^{2t_\star},\\
       %\tag{as $\sum_{t=1}^{t_\star}4^t\leq \frac{4^{t_\star+1}}{3}$}\\
        &\leq \frac{256c}{3}\cdot \frac{n+m-2}{\Delta_g^2}\cdot \log\left(\frac{4t_\star^3\cdot nm}{\delta}\right)\cdot\log (nm),
        %\tag{as $t_\star\leq \log_2(1/\Delta_g)+3$}.
    \end{align*}
    where the last inequality is due to $t_\star\leq \log_2(1/\Delta_g)+3$.
    The claim in the theorem follows from the fact that $\Delta_g=(\frac{n+m-2}{\mathbf{H_1}})^{1/2}$. 
\end{proof}

% \section{Fixed Budget near-optimal Algorithm}
% Algorithm can be converted to fixed budget by maintaining $\log T$ different guesses of $\Delta_g$. However, we would require a  fixed budget algorithm whether an element is a PSNE or not. This can be done via a procedure which is similar to Sequential Halving. We would then have a error probability bound that scales roughly as $\exp(-\frac{T}{H_1\cdot \log T})$ (ignoring $\log$ terms).

\subsection{Dueling Bandits}\label{sec:dueling}
Recall that in the context of dueling bandits defined by the matrix $\mathbf{P}$, we have $\Delta_{i,j}=|1/2-\mathbf{P}_{i,j}|$. The lower bound to identify the Condorcet winner $\istar$, whenever it exists, scales roughly as $\sum_{i\neq \istar}\Delta_{i,\istar}^{-2}$. Now, an instance of dueling can be interpreted as a two-player zero-game on a matrix $A$ where $A=\mathbf{P}$. It is straightforward to see  that the PSNE of $A$ is $(\istar,\istar)$. \Cref{alg-meta} achieves the bound of  $\sum_{i\neq \istar}\Delta_{i,\istar}^{-2}$, up to log factors.
\section{EXPERIMENTS}\label{sec:experiments}
To illustrate the performance of \infalg/, we introduce the following class of $d\times d$ reward matrices, where we use bold to denote a vector of repeated entries:
\begin{align*}
A_{\texttt{hard}}&(\Delta_{\min}, \beta):= \\
&\left[
\begin{BMAT}(@){c:c}{c:c}
    0.5 & \begin{BMAT}(@)[3pt]{c:c}{c} 0.5 + \Delta_{\min} & (\bm{0.5 + \beta})^\top \end{BMAT} \\
    \begin{BMAT}(@)[3pt]{c}{c:c} 0.5 - \Delta_{\min} \\ \bm{0.5 - \beta}\end{BMAT} & 0.5I + U
\end{BMAT}
\right],
\end{align*}
where $U$ is the matrix with all ones on the (strictly) upper triangular entries and  $0<\Delta_{\min} < \beta$.
These instances are difficult for the class of learners that play two no-regret learners against themselves, as the second row and second column have high and low sums, respectively, causing $A_{2,2}$ to potentially appear optimal.
Note, this is also a dueling bandit instance, which is a special case of a matrix game.

As baselines, we compare \infalg/\footnote{The subroutines \cmidval/ and \rmidval/ are wasteful. We optimize them for empirical performance by using an aggressive procedure similar to successive halving.
Furthermore, a fixed budget of $T$ is achieved for \infalg/ by setting $\Delta\approx \left(\frac{n+m}{T}\right)^{1/2}$ and $\delta\approx\text{constant}$.} to \tsinf/ \citep{zimmert2021tsallis}, \expix/ \citep{neu2015explore}, \lucbg/ \citep{zhou2017identify}, and uniform random sampling for different matrices in $A_{\texttt{hard}}$.
Each algorithm is given a fixed budget $T$ of samples that it can draw from the input matrix and we measure its percentage of identifying the PSNE for the given value of $T$.
Code to replicate these experiments is available at \url{https://github.com/aistats2024-noisy-psne/midsearch}; supplementary experimental results can be found in \Cref{appendix:plot}.

\begin{figure*}[h!]
  \centering
  %\includesvg[width=\textwidth]{figures/h1_sweep_1.svg}
  \includegraphics[width=\textwidth]{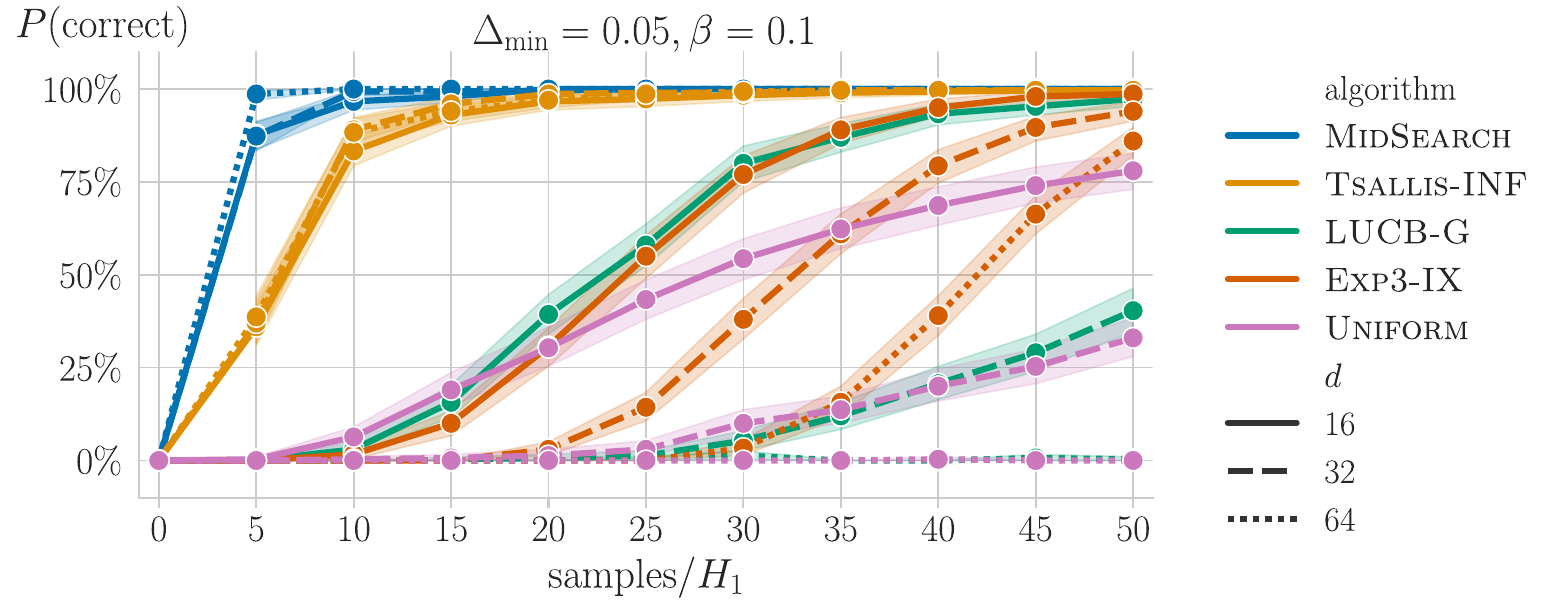}
  \caption{Plot for each algorithm on the $A_{\texttt{hard}}(0.05,0.1)$ instance with $N=300$ trials and with a budget of $50H_1$ samples; the $y$-axis shows what percentage of the time each algorithm's output was correct at different checkpoints during each run. We also overlay the corresponding Wilson binomial confidence interval ($p=95\%$).}
  \label{fig:exp_1}
\end{figure*}

\begin{figure*}[h!]
  %\includesvg[width=\textwidth]{figures/h1_sweep_3.svg}
  \includegraphics[width=\textwidth]{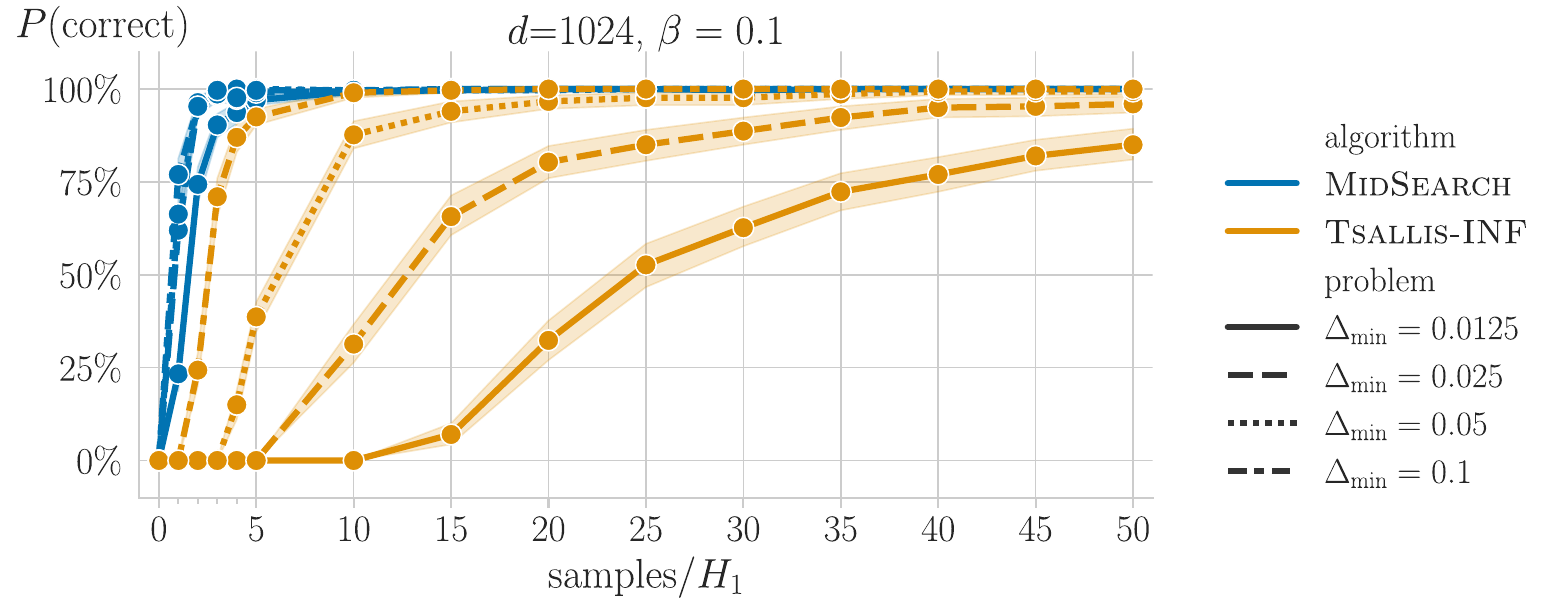}
  \caption{Plot of \infalg/ vs. \tsinf/ for varying $\Delta_{\min}$ parameter. Each algorithm was run for $N=300$ trials on the $A_{\texttt{hard}}(\Delta_{\min},0.1)$ instance ($d=1024$) with a budget of $50H_1$ samples.}
  \label{fig:exp_3}
\end{figure*}

Let $H_1=\frac{d-2}{\beta^2}+\frac{1}{\Delta_{\min}^2}$.
\Cref{fig:exp_1} shows that even with a budget of $50H_1$ samples, the \expix/, \lucbg/\footnote{For \lucbg/, we set $\delta\approx\text{constant}$ and return the PSNE it identifies by using at most $T$ samples.
We take the PSNE of the empirical mean matrix (if it exists) if the budget has been reached.},  and \textsc{Uniform} algorithms already begin to degrade in performance at $d=32$, returning the correct PSNE on $282$, $121$, and $99$ of the $300$ trials, respectively (\expix/ degrades slightly slower, worsening to 150/300 at $d=128$).

Thus, for larger matrices, we only compare \infalg/ and \tsinf/.
In \Cref{fig:exp_3} we plot the performance of \infalg/ vs. \tsinf/ for instances of $A_{\texttt{hard}}(\Delta_{\min}, 0.1)$ with $d=1024$, varying only $\Delta_{\min}$. We see that after $T=50H_1$ samples, \infalg/ remains close to perfect for the $d=1024$ matrix, while the accuracy of \tsinf/ decreases as $\Delta_{\min}$ decreases from $0.1$ to $0.0125$.

Note that the sample complexity for \textsc{Uniform} is $d^2 \Delta_{\min}^{-2}$, \lucbg/ is no worse than $d \Delta_{\min}^{-2} + d^2$, \expix/ is no worse than $d \Delta_{\min}^{-2}$, \tsinf/ is no worse than $d \Delta_{\min}^{-1} \beta^{-1} + \Delta_{\min}^{-2}$; \infalg/ is no worse than $\mathbf{H}_1 = d \beta^{-2} + \Delta_{\min}^{-2}$. Based on these experiments, we conjecture that these bounds are indeed tight.

\section{CONCLUSION}
In this paper we generalized stochastic bandits and dueling bandits by studying the problem of identifying the PSNE in two-player zero-sum games, whenever it exists. 
We designed an algorithm that matches the lower bound from \citet{zhou2017identify}, up to log factors. 
To the best of our knowledge, this is also the first algorithm that achieves near-optimal guarantees for identifying Condorcet winner in dueling bandits. 
Our work leads to the following interesting problems. First, can one design a $\delta$-PAC learner that has optimal bounds even in the log factors, similar to the extensive work in multi-armed bandits?
%Next, is there an algorithm that obtains a regret of $\frac{\log(T/\delta)}{\Delta_{i,\istar}}$ with probability at least $1-\delta$ in dueling bandits with Condorcet winner. \kevin{We did not define what we mean by regret here}
%Currently, the bounds by \citet{saha2022versatile} hold only in expectation.
%\kevin{And in some paper there is evidence that tsalis-inf necessarily cannot achieve high probability}
Although PSNE is not guaranteed to exist in every matrix, a mixed strategy Nash Equilibrium is always guaranteed to exist. The optimal sample complexity to identify an $\varepsilon$-Nash Equilibrium in an interesting open problem (see \citet{maiti2023instance} for more discussion). Last-iterate convergence in matrix games has also gained significant interest in recent years (see \citet{weilinear}). An important open problem here is to study last-iterate convergence in matrix games with noisy bandit feedback. Finally, can one design an algorithm for dueling bandits that simultaneously has optimal sample complexity for identifying the Condorcet winner, when it exists, and optimal regret guarantees with respect to the Condorcet winner (see \citet{saha2022versatile} for the definition and near-optimal result for regret), up to log factors.

%Finally, can one establish tight lower and upper bound for determining the support of the Nash Equilibrium of an arbitrary $n\times m$ zero-sum matrix game with noise.
%\kevin{Instead of this last sentence say somehting like "we assumed the existence of a PSNE. Not every matrix is guarnteed to have a PSNE, but a mixed strategy is always guaranteed. The optimal sample complexity is open for this problem (see Maiti 2022 for more discussion).}

%\kevin{We should state what we think the sample complexities of these algorithms are. e.g., uniform is $n^2 \Delta_{\min}^{-2}$, LUCB is no worse than $n \Delta_{\min}^{-2} + n^2$, \expix/ is no worse than $n \Delta_{\min}^{-2}$, Tsallis is no worse than $n \Delta_{\min}^{-1} \beta^{-1} + \Delta_{\min}^{-2}$. our algorithm is no worse than $\mathbf{H}_1 = n \beta^{-2} + \Delta_{\min}^{-2}$}
\section*{ACKNOWLEDGEMENTS}
This work was supported in part by NSF RI Award \# 1907907 and a Northrop Grumman University Research Award.
\bibliographystyle{plainnat}
\bibliography{refs.bib}
\clearpage
 
\appendix
\section{Technical Lemmas}
\begin{lemma}[sub-Gaussian tail bound]
Let $X_1,X_2,\ldots,X_n$ be i.i.d samples  from a $1$-sub-Gaussian distribution with mean $\mu$. Then we have the following for any $0<\delta<1$:
\begin{equation*}
    \mathbb{P}\left[\left|\frac{1}{n}\cdot\sum_{i=1}^nX_i-\mu\right|\geq \sqrt{\frac{2\log(2/\delta)}{n}}\right]\leq \delta
\end{equation*}
\end{lemma}
\begin{lemma}[Chernoff Bound]
    Let $X_1,X_2,\ldots,X_n$ be i.i.d samples  from a Bernoulli distribution with mean $\mu$. Then we have the following for any $0<\delta<1$:
\begin{equation*}
    \mathbb{P}\left[\frac{1}{n}\cdot\sum_{i=1}^nX_i\geq (1+\delta) \mu\right]\leq  e^{-\frac{n\mu\delta^2}{3}} \quad\text{and}\quad \mathbb{P}\left[\frac{1}{n}\cdot\sum_{i=1}^nX_i\leq (1-\delta) \mu\right]\leq  e^{-\frac{n\mu\delta^2}{2}} 
\end{equation*}
\end{lemma}

\section{Stochastic Multi-Armed Bandit}\label{appendix:MAB}
In stochastic Multi-Armed Bandit (MAB), there are $n$ arms where each arm $i$ is associated with a sub-distribution $\calD_i$ whose mean is $\mu_i$. When an arm $i$ is pulled (or sampled) by a learner, it observes the random variable $X_i\sim \calD_i$. The arms with the highest mean is referred as the best-arm. In pure exploration, one aims to design a learner that aims to find the best arm with probability at least $1-\delta$ by pulling as few arms as possible.

Let $\istar=\max_{i}\mu_i$ and $\Delta_i=\mu_{\istar}-\mu_i$. Optimal bounds to identify the best-arm scales roughly as $\sum_{i\neq \istar}\frac{1}{\Delta_i^2}\cdot \log(1/\delta)$. Now, an instance of MAB can be interpreted as a two-player zero-game on a matrix $A\in\mathbb{R}^{n\times 1}$ where $A_{i,1}=\mu_i$. It is easy to observe that the PSNE of $A$ is $(\istar,1)$. It is easy to observe that our algorithm essentially becomes a variant of Sequential Halving for this special case of MAB and achieves the bound of  $\sum_{i\neq \istar}\frac{1}{\Delta_i^2}\cdot \log(1/\delta)$, up to log factors.

\section{\rmidval/ Subroutine }\label{appendix:rmidval}

Here, we define \rmidval/ which solves the subproblem of finding a value near the median of means of $n$ arms.

Consider $\varepsilon,\delta>0$.
Given a set of $n$ arms $\calA$ with means $\mu_1\leq \mu_2\leq \ldots \leq \mu_n$, \rmidval/ outputs a value $\widehat v \in [\mu_{n/4+1}-\varepsilon,\mu_{n/2}+\varepsilon]$ with probability $1-\delta$ by using at most $O(\frac{1}{\varepsilon^2}\log\left(\frac{1}{\delta}\right))$ samples. The proof of this claim follows from an analogous argument as that of \cmidval/. 

\begin{algorithm}[ht]
\caption{\rmidval/($\calA,\varepsilon,\delta$)}
\begin{algorithmic}[1]
\STATE $\ell\gets \lceil 14\log(\frac{1}{\delta})\rceil$, $\delta_1\gets 0.05$, $\delta_2\gets 0.05$, $k\gets \left\lceil 108\log(\frac{4}{\delta_2})\right\rceil$ and $z\gets \frac{k}{3}+1$.
\FOR{$i=1$ to $\ell$}
\STATE $\calB_i \gets \{k $ arms sampled independently and uniformly at random from $\calA \}$.
\STATE Sample each arm $j\in \calB_i$ for $\left\lceil\frac{2\log(\frac{2k}{\delta_1})}{\varepsilon^2}\right\rceil$ times and compute the empirical mean $\widehat \mu_j$.
\STATE $v_i\gets $ $z$-th lowest value in $\{\widehat \mu_j:j\in\calB_i\}$
\ENDFOR
\STATE \textbf{return} the median of $\{v_i:i\in[\ell]\}$
\end{algorithmic}
\label{alg-rowmidval}
\end{algorithm}

\section{Verification algorithm}\label{appendix:verify}

In this section, we focus on the following subproblem:

\textbf{Subproblem:} Given an entry $(\widehat i,\widehat j)$ and a parameter $\Delta$, do the following with probability $1-\delta$:
\begin{itemize}
    \item If $\Delta \leq \Delta_g$ and $(\widehat i,\widehat j)=(\istar,\jstar)$, then $(\widehat i,\widehat j)$ is ``accepted'' as the PSNE of $A$
    \item If $(\widehat i,\widehat j)\neq (\istar,\jstar)$, then $(\widehat i,\widehat j)$ is ``not accepted'' as the PSNE of $A$.
\end{itemize}

Consider a stochastic multi-armed bandit instance with means $\mu_1\geq \mu_2\geq \ldots \geq \mu_n$. Let $H_{0}=\sum_{i=2}^n\frac{1}{(\mu_1-\mu_i)^2}$.
Let $\texttt{Alg}$ be a best arm identification algorithm that does the following:
\begin{itemize}
    \item If the best arm is unique, $\texttt{Alg}$ finds the best arm by using at most $c_1\cdot H_{0}\cdot \log(\frac{\log H_0}{\delta})$ samples with probability at least $1-\delta/2$ where $c_1$ is an absolute constant. 
    \item If the best arm is not unique, $\texttt{Alg}$ does not terminate with probability at least $1-\delta/2$.
\end{itemize}
See \citet{karnin2013almost} for one such algorithm. We use $\texttt{Alg}$ to design a procedure called \textsc{Verify} that achieves the desired guarantees for the above the sub-problem. 

\begin{algorithm}[ht]
\caption{\textsc{Verify}($A,\widehat i,\widehat j,\delta,\Delta)$}
\begin{algorithmic}[1]
\STATE Let $\nu_1$ be a multi-armed bandit instance with means $\mu_i=A_{i,\widehat j}$ for all $i\in [n]$
\STATE Let $\nu_2$ be a multi-armed bandit instance with means $\mu_j=-A_{\widehat i,j}$ for all $j\in [m]$
\STATE $H_{\star}\gets \frac{n+m-2}{\Delta^2}$
\STATE Run \texttt{Alg} on the bandit instance $\nu_1$ and terminate it after it takes $c_1\cdot H_{\star}\cdot \log(\frac{\log H_\star}{\delta})$ samples.\label{st:verify-1}
\STATE If \texttt{Alg} terminates without returning arm $\widehat i$ as the best arm, then return ``not accepted''.
\STATE Run \texttt{Alg} on the bandit instance $\nu_2$ and terminate it after it takes $c_1\cdot H_{\star}\cdot \log(\frac{\log H_\star}{\delta})$ samples.\label{st:verify-2}
\STATE If \texttt{Alg} terminates without returning arm $\widehat j$ as the best arm, then return ``not accepted''.
\STATE Return ``accepted''.
\end{algorithmic}
\label{alg-verify}
\end{algorithm}

We first establish the sample complexity of \textsc{Verify} in the following lemma.
\begin{lemma}\label{verify-sample}
    \textsc{Verify} takes at most $2c_1\cdot H_{\star}\cdot \log(\frac{\log H_\star}{\delta})$ samples before terminating
\end{lemma}
\begin{proof}
    The lemma trivially follows from lines \ref{st:verify-1} and \ref{st:verify-2} of \textsc{Verify}.
\end{proof}

We prove the following proposition that is useful to establish the correctness of our algorithm
\begin{proposition}\label{prop-verify}
    If $(\widehat i,\widehat j)\neq (\istar,\jstar)$, then either $A_{\widehat i,\widehat j}<A_{\istar,\widehat j}$ or $A_{\widehat i,\widehat j}>A_{\widehat i,\jstar}$.
\end{proposition}
\begin{proof}
    Let us assume that $A_{\widehat i,\widehat j}\geq A_{\istar,\widehat j}$ and $A_{\widehat i,\widehat j}\leq A_{\widehat i,\jstar}$. This implies that $A_{\istar,\widehat j}\leq A_{\widehat i,\jstar}$ which is a contradiction as $A_{\widehat i,\jstar} < A_{\istar,\jstar} <A_{\istar,\widehat j}$.
\end{proof}

The following lemma establishes the correctness of \textsc{Verify}.
\begin{lemma}\label{verify-correct}
    \textsc{Verify} does the following with probability $1-\delta$:
\begin{itemize}
    \item If $\Delta \leq \Delta_g$ and $(\widehat i,\widehat j)=(\istar,\jstar)$, then $(\widehat i,\widehat j)$ is ``accepted'' as the PSNE of $A$
    \item If $(\widehat i,\widehat j)\neq (\istar,\jstar)$, then $(\widehat i,\widehat j)$ is ``not accepted'' as the PSNE of $A$.
\end{itemize}
\end{lemma}
\begin{proof}
    Let us first consider the case when $\Delta \leq \Delta_g$ and $(\widehat i,\widehat j)=(\istar,\jstar)$. Recall the bandit instances $\nu_1$ and $\nu_2$. 
    Let $H_{x}=\sum_{i\in[n]\setminus \{\istar\}}\frac{1}{(A_{\istar,\jstar}-A_{i,\jstar})^2}$ and $H_{y}=\sum_{j\in[m]\setminus \{\jstar\}}\frac{1}{(A_{\istar,\jstar}-A_{\istar,j})^2}$. Observe that $H_x+H_y=\frac{n+m-2}{\Delta_g^2}\leq \frac{n+m-2}{\Delta^2}=H_\star$.  Now observe that $\istar$ is the best arm of the instance $\nu_1$ and $\jstar$ is the best arm of the instance $\nu_2$. Hence, \texttt{Alg} does the following with probability $1-\delta$:
    \begin{itemize}
        \item For the instance $\nu_1$, \texttt{Alg} returns $\istar$ as the best arm after taking at most $c_1\cdot H_{x}\cdot \log(\frac{\log H_x}{\delta})$ samples.
        \item For the instance $\nu_2$, \texttt{Alg} returns $\istar$ as the best arm after taking at most $c_1\cdot H_{y}\cdot \log(\frac{\log H_y}{\delta})$ samples.
    \end{itemize}
    Hence, we ``accept'' $(\widehat i,\widehat j)$ with probability at least $1-\delta$.

    Next, we consider the case when $(\widehat i,\widehat j)\neq (\istar,\jstar)$. Due to \Cref{prop-verify}, either $\widehat i$ is not the best arm of the instance $\nu_1$ or $\widehat j$ is not the best of arm of the instance $\nu_2$. Hence, with probability $1-\delta$, $(\widehat i,\widehat j)$ is ``not accepted''.
\end{proof} 

\section{Omitted Calculations} \label{appendix:omit}
Here we state the calculations omitted from various proofs in the main body.

\begin{proof}[Omitted Calculations from the proof of \Cref{midval:correct1}]
  Recall that we generate the set $\calB_i$ by sampling $k$ arms independently and uniformly at random from $\calA$ where $k=\lceil 108\log(\frac{4}{\delta_2})\rceil=474$. Hence, $\frac{k}{3}+1$ is an integer equal to 159 and $\frac{7k}{20}>165$. Also recall that $\mathbb{E}[n_{1,i}]=\mathbb{E}[n_{2,i}]=k/4$. First we have the following due to the Chernoff bound:
  \begin{align*}
      \mathbb{P}\left[n_{1,i}\leq \frac{7k}{40}\right]&= \mathbb{P}\left[n_{1,i}\leq \left(1-\frac{3}{10}\right)\cdot\frac{k}{4}\right]\\
      &=\mathbb{P}\left[n_{1,i}\leq \left(1-\frac{3}{10}\right)\cdot\mathbb{E}[n_{1,i}]\right]\\
      &\leq \exp\left(-\frac{\mathbb{E}[n_{1,i}]\cdot (0.3)^2 }{2}\right)\\
      &= \exp\left(-\frac{k\cdot 0.09}{4\cdot 2}\right)\\
      &\leq \exp\left(-\frac{108\log(\frac{4}{\delta_2})\cdot 0.09}{4\cdot 2}\right)\\
      &=  \exp\left(-1.215\cdot\log\left(\frac{4}{\delta_2}\right)\right)\\
      &<\exp\left(-\log\left(\frac{4}{\delta_2}\right)\right)=\frac{\delta_2}{4}
  \end{align*}
  Next we have the following due to the following due to the Chernoff bound
  \begin{align*}
      \mathbb{P}\left[n_{1,i}\geq \frac{k}{3}\right]&= \mathbb{P}\left[n_{1,i}\geq \left(1+\frac{1}{3}\right)\cdot\frac{k}{4}\right]\\
      &=\mathbb{P}\left[n_{1,i}\geq \left(1+\frac{1}{3}\right)\cdot\mathbb{E}[n_{1,i}]\right]\\
      &\leq \exp\left(-\frac{\mathbb{E}[n_{1,i}]\cdot (1/3)^2 }{3}\right)\\
      &= \exp\left(-\frac{k}{4\cdot 3\cdot 9}\right)\\
      &\leq \exp\left(-\frac{108\log(\frac{4}{\delta_2})}{108}\right)\\
      &=\exp\left(-\log\left(\frac{4}{\delta_2}\right)\right)=\frac{\delta_2}{4}
  \end{align*}
  Similarly, we can show using Chernoff bound that $\frac{7k}{40}\leq |n_{2,i}|\leq \frac{k}{3}$ occurs with probability at least $1-\frac{\delta_2}{2}$.

  Recall the definition of $Z$. Also recall that $\ell=\lceil 14\log(\frac{1}{\delta}) \rceil$ and $\mathbb{E}[Z]\geq\frac{9\ell}{10}$. Now we have the following due to the Chernoff bound:
   \begin{align*}
      \mathbb{P}\left[Z\leq 0.54\ell\right]&= \mathbb{P}\left[n_{1,i}\leq \left(1-\frac{2}{5}\right)\cdot\frac{9\ell}{10}\right]\\
      &\leq\mathbb{P}\left[n_{1,i}\leq \left(1-\frac{2}{5}\right)\cdot\mathbb{E}[Z]\right]\\
      &\leq \exp\left(-\frac{\mathbb{E}[Z]\cdot (0.4)^2 }{2}\right)\\
      &\leq \exp\left(-\frac{9\ell\cdot 0.16}{10\cdot 2}\right)\\
      &\leq \exp\left(-\frac{9\cdot 14\log(\frac{1}{\delta})\cdot 0.16}{10\cdot 2}\right)\\
      &=  \exp\left(-1.008\cdot\log\left(\frac{1}{\delta}\right)\right)\\
      &<\exp\left(-\log\left(\frac{1}{\delta}\right)\right)=\delta
  \end{align*}
\end{proof}

\begin{proof}[Omitted calculations in the proof of \Cref{theorem-main}]
    Using \Cref{lem:psne-sample} we get that at any iteration $t$, \infalg/ uses at most $c_1\cdot \frac{n+m-2}{\Delta_t^2}\cdot\log(\frac{nm}{\delta_t})\cdot \log(nm)$ samples where $c_1$ is an absolute constant. Using \Cref{verify-guarantees} we get that at any iteration $t$, \textsc{Verify} uses at most $c_2\cdot \frac{n+m-2}{\Delta_t^2}\cdot \log\left(\frac{\log\left(\frac{n+m}{\Delta_t^2}\right)}{\delta_t}\right)$ where $c_2$ is absolute constant. As $n,m\geq 1$, we have $n\geq 1\geq \frac{m}{2m-1}$ which implies $n+m\leq 2nm$. Now we have the following:
    \begin{align*}
        \log\left(\frac{\log\left(\frac{n+m}{\Delta_t^2}\right)}{\delta_t}\right)&\leq\log\left(\frac{\log\left(\frac{2nm}{\Delta_t^2}\right)}{\delta_t}\right)\\ 
        &\leq \log\left(\frac{\log(2nm)+2\log\left(\frac{1}{\Delta_t}\right)}{\delta_t}\right)\\ 
        &\leq \log\left(\frac{2nm+2nm\log\left(\frac{1}{\Delta_t}\right)}{\delta_t}\right)\tag{as $\log(x)<x$}\\ 
        &=\log\left(\frac{2nm\log\left(\frac{e}{\Delta_t}\right)}{\delta_t}\right)\\ 
    \end{align*}
    Now the total number of samples $\tau$ taken at iteration $t$ is upper bounded as follows:
    \begin{align*}
        \tau&=c_1\cdot \frac{n+m-2}{\Delta_t^2}\cdot\log\left(\frac{nm}{\delta_t}\right)\cdot \log(nm)+c_2\cdot \frac{n+m-2}{\Delta_t^2}\cdot \log\left(\frac{\log\left(\frac{n+m}{\Delta_t^2}\right)}{\delta_t}\right)\\
        &\leq c_1\cdot \frac{n+m-2}{\Delta_t^2}\cdot\log\left(\frac{nm}{\delta_t}\right)\cdot \log(nm)+c_2\cdot \frac{n+m-2}{\Delta_t^2}\cdot \log\left(\frac{2nm\log\left(\frac{e}{\Delta_t}\right)}{\delta_t}\right)\\ 
        &\leq c\cdot \frac{n+m-2}{\Delta_t^2}\cdot\log\left(\frac{nm\log\left(\frac{1}{\Delta_t}\right)}{\delta_t}\right)\cdot \log(nm)
    \end{align*}
    where $c$ is an absolute constant.
\end{proof}

\section{Additional Plots}\label{appendix:plot}
In this section, we perform few additional experiments. We consider the hard instance $A_{\texttt{hard}}(\Delta_{\min},\beta)$ described in \Cref{sec:experiments} with $d=128,256,512$. As shown in \Cref{sec:experiments} both \infalg/ and \tsinf/ perform far better than \expix/, \lucbg/ and uniform random sampling. Hence, we compare the performance of \infalg/ and \tsinf/ here by measuring the probability of identifying the PSNE. First in \Cref{fig:exp_2}, we fix $\beta=0.1$ and vary $\Delta_{\min}$ and compare the performance of \infalg/ and \tsinf/. We observe that \infalg/ has a better performance than \tsinf/. Next in \Cref{fig:exp_5}, we fix $\Delta_{\min}=0.01$ and vary $\beta$ and compare the performance of \infalg/ and \tsinf/. We observe that \infalg/ has a better performance than \tsinf/. Finally in \Cref{fig:exp_6}, we vary both $\beta$ and $\Delta_{\min}$ and compare the performance of \infalg/ and \tsinf/. We observe that \infalg/ has a better performance than \tsinf/. 
\begin{figure*}[h!]
  %\includesvg[width=\textwidth]{figures/h1_sweep_2.svg}
  \includegraphics[width=\textwidth]{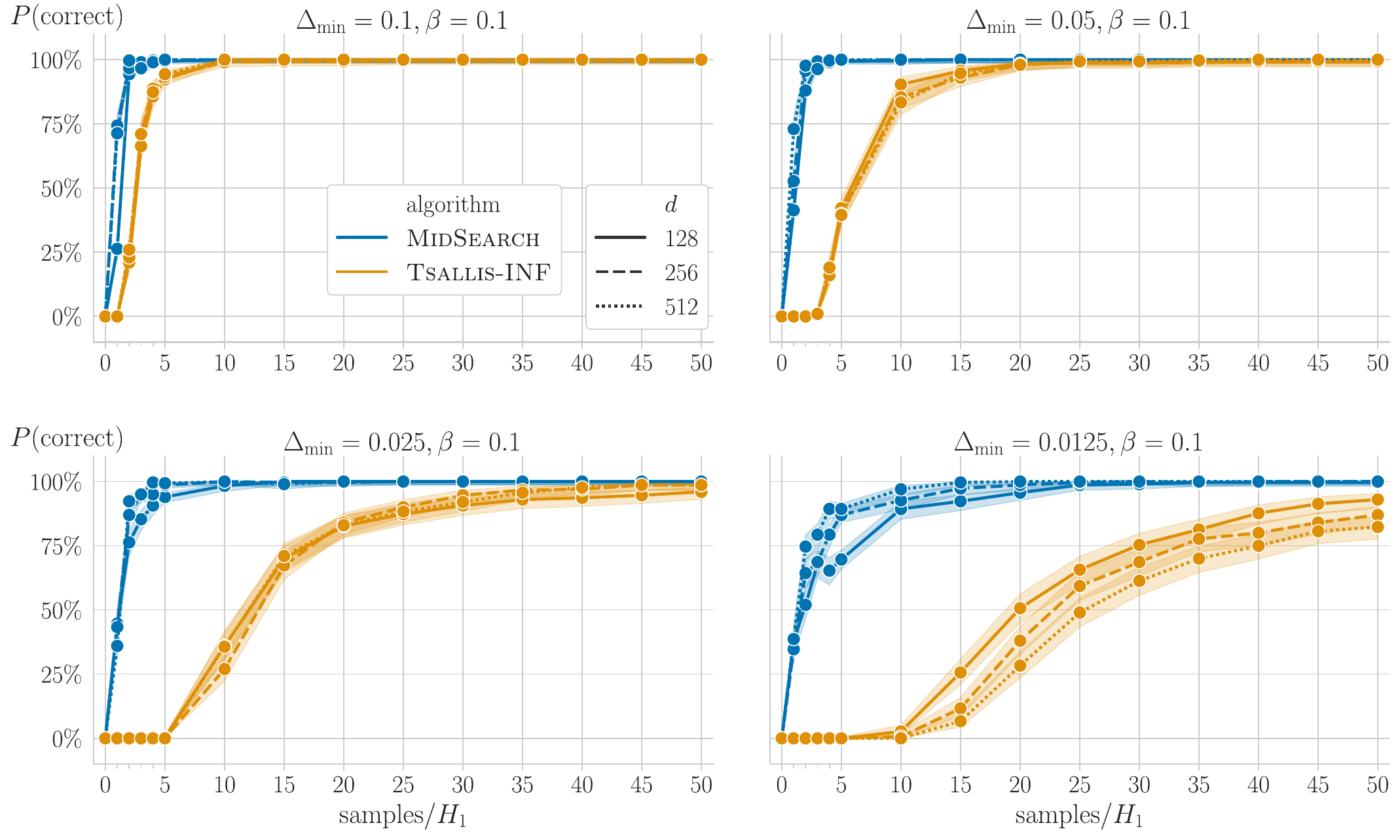}
  \caption{Plot of \infalg/ vs. \tsinf/ for varying $\Delta_{\min}$ parameter. Each algorithm was run for $N=300$ trials on the $A_{\texttt{hard}}(\Delta_{\min},0.1)$ instance with a budget of $50H_1$ samples.}
  \label{fig:exp_2}
\end{figure*}

\begin{figure*}[h!]
  %\includesvg[width=\textwidth]{figures/h1_sweep_5.svg}
  \includegraphics[width=\textwidth]{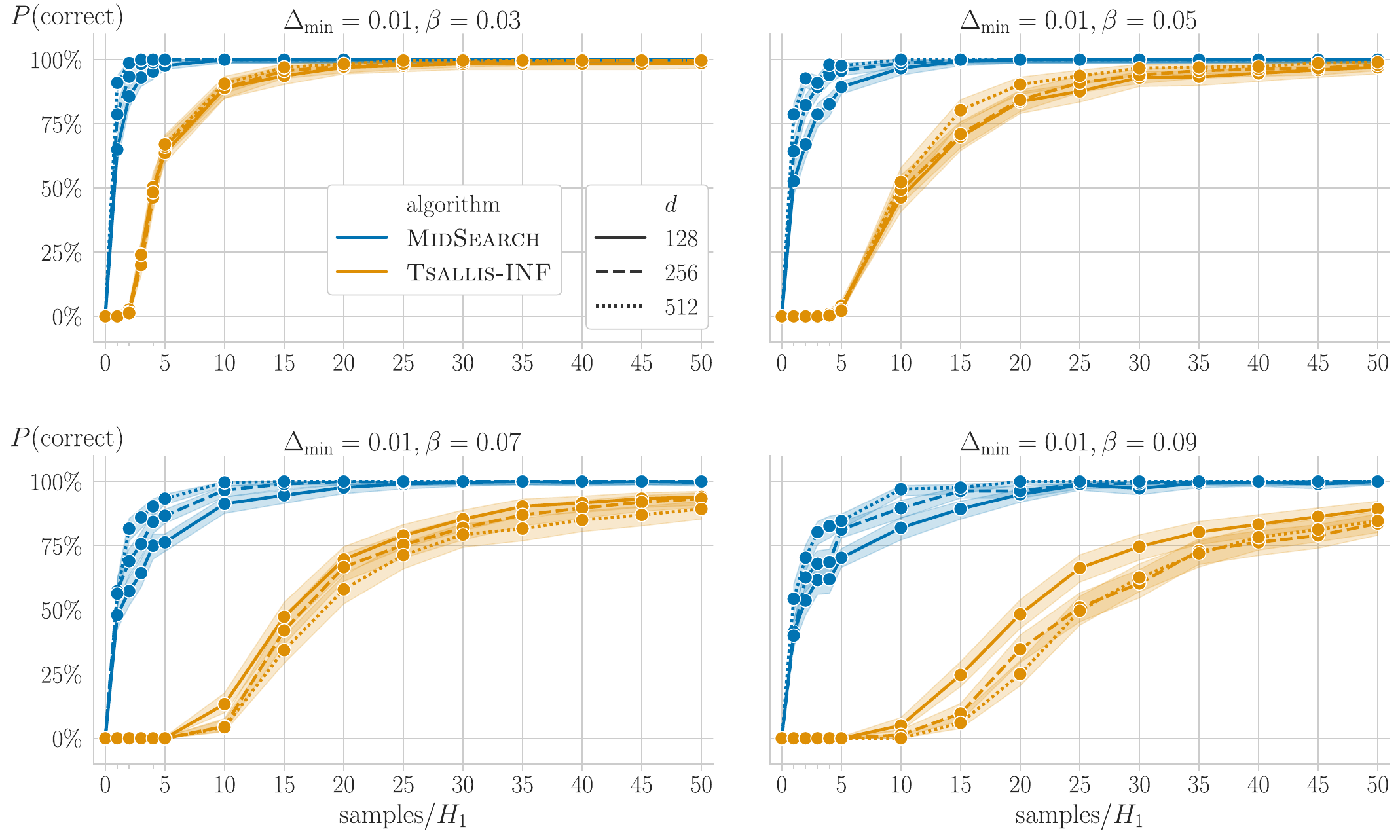}
  \caption{Plot of \infalg/ vs. \tsinf/ for varying $\beta$ parameter. Each algorithm was run for $N=300$ trials on the $A_{\texttt{hard}}(0.01,\beta)$ instance with a budget of $50H_1$ samples.}
  \label{fig:exp_5}
\end{figure*}

\begin{figure*}[h!]
  %\includesvg[width=\textwidth]{figures/h1_sweep_6.svg}
  \includegraphics[width=\textwidth]{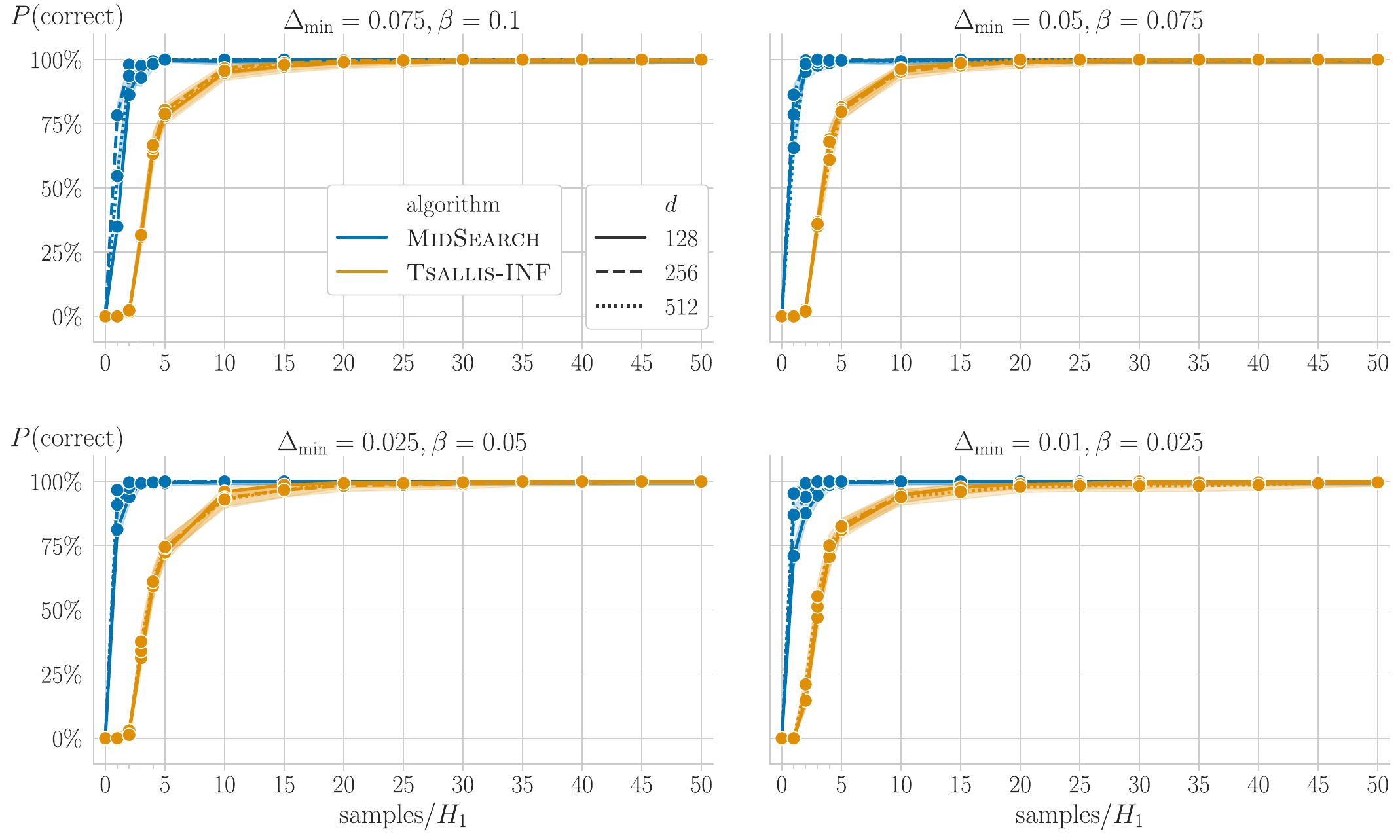}
  \caption{Plot of \infalg/ vs. \tsinf/ for varying both $\Delta_{\min}$ and $\beta$ parameters. Each algorithm was run for $N=300$ trials on the $A_{\texttt{hard}}(\Delta_{\min},\beta)$ instance with a budget of $50H_1$ samples.}
  \label{fig:exp_6}
\end{figure*}

%\clearpage

\end{document}